\documentclass[sigconf, nonacm]{acmart}
\usepackage{graphicx} % Required for inserting images

\usepackage{ amsthm, amssymb}
\newtheorem{proposition}{Proposition}

\usepackage[ruled,vlined]{algorithm2e}
\usepackage{array}
\usepackage{tikz}
\usetikzlibrary{arrows.meta,positioning}
\usepackage[most]{tcolorbox}
\usepackage{mathdots}
\usepackage{booktabs}
\newcommand{\eat}[1]{}
\usepackage{multirow}
\usetikzlibrary{arrows.meta,positioning,fit,calc,matrix}

\newcommand\vldbyear{2026}
\newcommand\vldbworkshop{Applied AI for Database Systems and Applications (AIDB)}
\newcommand\vldbauthors{\authors}
\newcommand\vldbtitle{\shorttitle} 
\newcommand\vldbavailabilityurl{https://github.com/alanganyDB/Structural-Adversarial-Attacks-on-Relational-Deep-Learning-under-Integrity-Constraints.git}
\newcommand\vldbpagestyle{empty} 

\begin{document}
\title{Structural Adversarial Attacks on Relational Deep Learning under Integrity Constraints}

%%
%% The "author" command and its associated commands are used to define the authors and their affiliations.
\author{Alan Gany}
\affiliation{%
  \institution{Univ.\ Grenoble Alpes, CNRS, Grenoble INP, LIG}
  \city{Grenoble}
  \country{France}
  \postcode{38000}
}
\email{alan.gany@univ-grenoble-alpes.fr}

\author{Bogdan Cautis}
\affiliation{%
  \institution{Singapore Institute of Technology}
  \country{Singapore}
}
\email{bogdan.cautis@singaporetech.edu.sg}

\author{Silviu Maniu}
\affiliation{%
  \institution{Univ.\ Grenoble Alpes, CNRS, Grenoble INP, LIG}
  \city{Grenoble}
  \country{France}
  \postcode{38000}
}
\email{silviu.maniu@univ-grenoble-alpes.fr}

\begin{abstract}
Relational Deep Learning (RDL) has become a standard methodology for
machine learning on relational databases: the database is encoded as a
heterogeneous temporal graph in which tuples become nodes and primary-key
to foreign-key (PK--FK) dependencies become typed edges, over which a
graph neural network is trained for downstream prediction. We study the
adversarial robustness of this pipeline. We consider a white-box attacker
who knows how the graph is built and the model is trained, reasons about
perturbations on the graph, but can only act on the upstream database --
by rewiring foreign-key references while preserving the integrity
constraints of the schema (foreign-key validity, the degree-one FK
constraint, and functional dependencies). This restricts the attacker to
a constrained, combinatorial set of admissible edits under a global
perturbation budget, which is intractable to explore exhaustively and
made non-additive by GNN message passing. We investigate seven attack heuristics -- two random sampling baselines and five gradient-guided variants that exploit differentiable edge masks -- and evaluate them on the RelBench \texttt{rel-f1} benchmark. Gradient-based attacks consistently outperform
random baselines on regression tasks, whereas gains on classification are
smaller, which we attribute to low label-flip rates and greater local
stability of classification outputs.
\end{abstract}

\maketitle

\pagestyle{\vldbpagestyle}
\begingroup\small\noindent\raggedright\textbf{VLDB Workshop Reference Format:}\\
\vldbauthors. \vldbtitle. VLDB \vldbyear\ Workshop: \vldbworkshop.\\ %\vldbvolume(\vldbissue): \vldbpages, \vldbyear.\\
%\href{https://doi.org/\vldbdoi}{doi:\vldbdoi}
\endgroup
\begingroup
\renewcommand\thefootnote{}\footnote{\noindent
This work is licensed under the Creative Commons BY-NC-ND 4.0 International License. Visit \url{https://creativecommons.org/licenses/by-nc-nd/4.0/} to view a copy of this license.
}\addtocounter{footnote}{-1}\endgroup

\ifdefempty{\vldbavailabilityurl}{}{
\vspace{.3cm}
\begingroup\small\noindent\raggedright\textbf{VLDB Workshop Artifact Availability:}\\
The source code, data, and/or other artifacts have been made available at \url{\vldbavailabilityurl}.
\endgroup
}

\section{Introduction}
Deep neural networks achieve strong predictive performance across many
domains, yet they remain fragile: small, carefully crafted perturbations
of their inputs can drastically alter their predictions, exposing these
systems to malicious manipulation. Adversarial attacks, and in particular
gradient-based ones, have been demonstrated on a wide range of data
structures, including images, tabular data, graphs, and knowledge graphs.
Attacks on dense inputs such as images are comparatively easy to
formulate, whereas attacks on sparse, discrete structures are
substantially harder because of their combinatorial nature; gradient-based
methods nonetheless remain effective in these settings.

Relational databases store large volumes of structured, high-value
information, which makes them attractive targets. They have, however,
remained largely outside the scope of adversarial machine learning, simply
because few learning methods operated directly on relational data.
Relational Deep Learning (RDL) changes this
picture~\cite{fey2024position, robinson2024relbench, gu2026relbench}. RDL
encodes a multi-table database as a heterogeneous temporal graph in which
each tuple becomes a node and each primary-key to foreign-key (PK--FK)
dependency becomes a typed edge, and it trains a graph neural network
end-to-end for a downstream prediction task. By turning the database
itself into the model's input, RDL also turns it into an attack surface.

We study the adversarial robustness of this pipeline under a white-box
threat model. The attacker knows how the graph is constructed from the
schema and how the model is trained, and can therefore reproduce both, and
reason about perturbations directly on the graph. The attacker cannot,
however, edit the graph itself: the graph is a derived artifact, and the
only actionable surface is the upstream database. Every perturbation must
therefore be expressed as an edit to the database that leaves it
consistent, i.e., that does not violate the integrity constraints of the
schema. An illustration of this setting is provided in
Figure~\ref{fig:db_to_poisoned_prediction}.

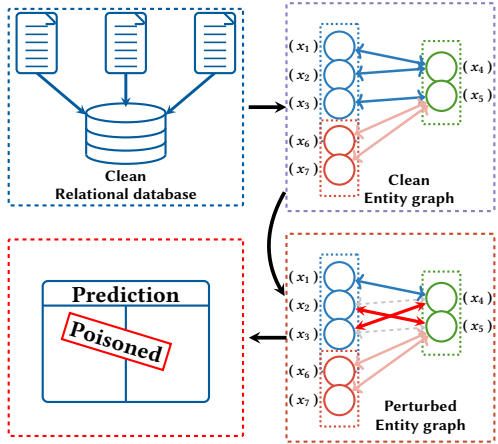
\begin{figure}[t]
\centering

\scalebox{0.5}{%

\begin{tikzpicture}[
    font=\sffamily\bfseries\Large,
    box/.style={draw=#1, dashed, ultra thick, rounded corners=1pt},
    doc/.style={draw=cleanblue, ultra thick, rounded corners=3pt, fill=white},
    db/.style={draw=cleanblue, ultra thick, fill=white},
    nodec/.style={circle, draw=#1, ultra thick, fill=white, minimum size=8mm},
    relbox/.style={draw=#1, dotted, ultra thick, rounded corners=1pt},
    arr/.style={->, line width=2.2pt, >=stealth},
    edge/.style={<->, line width=1.9pt, #1},
    ghost/.style={<->, line width=1.5pt, gray!45, dashed},
    rededge/.style={<->, line width=2.4pt, red, >=stealth},
    lab/.style={font=\sffamily\bfseries\Large},
    mathlab/.style={font=\boldmath\bfseries\Large}
]

% =================================================
% Colors
% =================================================

\definecolor{cleanblue}{HTML}{005A9C}
\definecolor{softblue}{HTML}{2B7BBB}
\definecolor{softgreen}{HTML}{4C9A2A}
\definecolor{softred}{HTML}{D94A38}
\definecolor{purplebox}{HTML}{8E7CC3}
\definecolor{orangebox}{HTML}{C44E2B}

% =================================================
% Panel 1: Clean relational database
% =================================================

\node[
    box=cleanblue,
    minimum width=6.2cm,
    minimum height=5.2cm
] (dbbox) at (0,3) {};

% -------------------------------------------------
% Files
% -------------------------------------------------

\begin{scope}[shift={(0,3.90)}, scale=1.18]

    \foreach \x in {-2,0,2}{

        \begin{scope}[shift={(\x,0)}]

            \draw[doc] (-0.45,0) rectangle (0.45,1.35);

            \draw[doc]
                (0.20,1.35)
                --
                (0.45,1.10)
                --
                (0.20,1.10);

            \foreach \j in {0,...,5}{
                \draw[cleanblue, very thick]
                    (-0.25,1.05-0.18*\j)
                    --
                    (0.25,1.05-0.18*\j);
            }

        \end{scope}
    }

\end{scope}

% -------------------------------------------------
% Database
% -------------------------------------------------

\begin{scope}[shift={(0,2.50)}]

    % top ellipse
    \draw[db]
        (0,0.35)
        ellipse (1.05 and 0.25);

    % borders
    \draw[db] (-1.05,-0.75) -- (-1.05,0.35);
    \draw[db] ( 1.05,-0.75) -- ( 1.05,0.35);

    % bottom visible arc only
    \draw[db]
        (-1.05,-0.75)
        arc (-180:0:1.05 and 0.25);

    % visible middle rings
    \draw[cleanblue, ultra thick]
        (-1.05,-0.05)
        arc (180:360:1.05 and 0.25);

    \draw[cleanblue, ultra thick]
        (-1.05,-0.40)
        arc (180:360:1.05 and 0.25);

\end{scope}

% arrows
\draw[arr, cleanblue] (-2.35,3.90) -- (-1.05,2.80);
\draw[arr, cleanblue] (0,3.90) -- (0,3);
\draw[arr, cleanblue] (2.35,3.90) -- (1.05,2.80);

\node[
    align=center,
    font=\sffamily\bfseries\LARGE
]
at (0,0.95)
{Clean\\Relational database};

% =================================================
% Panel 2: Clean entity graph (licit: one parent per child)
% =================================================

\node[
    box=purplebox,
    minimum width=5.5cm,
    minimum height=5.5cm
] (cgbox) at (7.0,3) {};

\begin{scope}[shift={(7.0,3)}]

    \node[
        relbox=softblue,
        minimum width=1.0cm,
        minimum height=3.1cm
    ]
    at (-1.35,0.65) {};

    \node[
        relbox=softgreen,
        minimum width=1.0cm,
        minimum height=2.2cm
    ]
    at (1.35,0.65) {};

    \node[
        relbox=softred,
        minimum width=1.0cm,
        minimum height=2.0cm
    ]
    at (-1.35,-1.35) {};

    % blue nodes
    \foreach \y/\i in {1.60/1,0.85/2,0.10/3}{
        \node[
            nodec=softblue,
            label={[mathlab]left:{$(\,x_{\i}\,)$}}
        ]
        (cB\i)
        at (-1.35,\y) {};
    }

    % green nodes
    \foreach \y/\i in {1.05/4,0.30/5}{
        \node[
            nodec=softgreen,
            label={[mathlab]right:{$(\,x_{\i}\,)$}}
        ]
        (cG\i)
        at (1.35,\y) {};
    }

    % red nodes
    \foreach \y/\i in {-0.90/6,-1.65/7}{
        \node[
            nodec=softred,
            label={[mathlab]left:{$(\,x_{\i}\,)$}}
        ]
        (cR\i)
        at (-1.35,\y) {};
    }

    % r1 edges (blue children -> green parent): exactly one parent per child
    \draw[edge=softblue] (cB1) -- (cG4);
    \draw[edge=softblue] (cB2) -- (cG4);
    \draw[edge=softblue] (cB3) -- (cG5);

    % r2 edges (red children -> green parent)
    \draw[edge=softred!45] (cR6) -- (cG5);
    \draw[edge=softred!45] (cR7) -- (cG5);

\end{scope}

\node[
    align=center,
    font=\sffamily\bfseries\LARGE
]
at (7.5,0.80)
{Clean\\Entity graph};

% =================================================
% Panel 3: Perturbed entity graph (initial + 2 rewirings, still licit)
% =================================================

\node[
    box=orangebox,
    minimum width=5.5cm,
    minimum height=5.5cm
] (pgbox) at (7.0,-3.1) {};

\begin{scope}[shift={(7.0,-3.1)}]

    \node[
        relbox=softblue,
        minimum width=1.0cm,
        minimum height=3.1cm
    ]
    at (-1.35,0.65) {};

    \node[
        relbox=softgreen,
        minimum width=1.0cm,
        minimum height=2.2cm
    ]
    at (1.35,0.65) {};

    \node[
        relbox=softred,
        minimum width=1.0cm,
        minimum height=2.0cm
    ]
    at (-1.35,-1.35) {};

    % blue nodes
    \foreach \y/\i in {1.60/1,0.85/2,0.10/3}{
        \node[
            nodec=softblue,
            label={[mathlab]left:{$(\,x_{\i}\,)$}}
        ]
        (pB\i)
        at (-1.35,\y) {};
    }

    % green nodes
    \foreach \y/\i in {1.05/4,0.30/5}{
        \node[
            nodec=softgreen,
            label={[mathlab]right:{$(\,x_{\i}\,)$}}
        ]
        (pG\i)
        at (1.35,\y) {};
    }

    % red nodes
    \foreach \y/\i in {-0.90/6,-1.65/7}{
        \node[
            nodec=softred,
            label={[mathlab]left:{$(\,x_{\i}\,)$}}
        ]
        (pR\i)
        at (-1.35,\y) {};
    }

    % replaced edges, shown faded (old parents of x2 and x3)
    \draw[ghost] (pB2) -- (pG4);
    \draw[ghost] (pB3) -- (pG5);

    % unchanged edges
    \draw[edge=softblue] (pB1) -- (pG4);
    \draw[edge=softred!45] (pR6) -- (pG5);
    \draw[edge=softred!45] (pR7) -- (pG5);

    % rewired edges (each reassigns one child to a new parent; one parent per child preserved)
    \draw[rededge] (pB2) -- (pG5);
    \draw[rededge] (pB3) -- (pG4);

\end{scope}

\node[
    align=center,
    font=\sffamily\bfseries\LARGE
]
at (7.8,-5.2)
{Perturbed\\Entity graph};

% =================================================
% Panel 4: Prediction
% Text kept unchanged on purpose
% =================================================

\node[
    box=red,
    minimum width=6.2cm,
    minimum height=5.2cm
]
(predbox)
at (0,-3.1)
{};

\begin{scope}[shift={(0,-3.25)}]

    \draw[
        cleanblue,
        ultra thick,
        rounded corners=4pt
    ]
    (-2.2,-1.55)
    rectangle
    (2.2,1.65);

    \draw[cleanblue, ultra thick]
        (-2.2,1.00)
        --
        (2.2,1.00);

    \draw[cleanblue, ultra thick]
        (0,1.00)
        --
        (0,-1.55);

    \node[
        font=\bfseries\Huge
    ]
    at (0,1.32)
    {Prediction};

    \node[
        draw=red,
        ultra thick,
        rotate=-17,
        font=\bfseries\Huge,
        text=black,
        fill=white,
        inner sep=4pt
    ]
    at (-0.25,0.05)
    {Poisoned};

\end{scope}

% =================================================
% Global arrows
% =================================================

\draw[
    ->,
    line width=2.8pt,
    black,
    >=stealth
]
(3.25,3)
--
(4.25,3);

\draw[
    ->,
    line width=2.8pt,
    black,
    >=stealth
]
(4.2,0.75)
to[out=-120,in=120]
(4.2,-2.0);

\draw[
    ->,
    line width=2.8pt,
    black,
    >=stealth
]
(4.25,-3.1)
--
(3.25,-3.1);

\end{tikzpicture}

}

\caption{Overview of the attack setting. A relational database is encoded
into a heterogeneous entity graph following the RDL construction, where
tuples become nodes and PK--FK dependencies become typed edges. The
attacker rewires a small number of foreign-key edges while preserving the
schema's integrity constraints, so that every tuple keeps exactly one
parent, producing a perturbed entity graph that, once passed through the
trained GNN, degrades the downstream prediction. Rewired edges are shown
in red, and the edges they replace are faded.}
\label{fig:db_to_poisoned_prediction}
\end{figure}

Within this setting, two attack directions can be distinguished:
\begin{itemize}
    \item \emph{Feature attacks}, which target the heterogeneous feature
    representations derived from the database and assembled as tensor
    frames. Designing perturbations in this feature space is difficult due
    to the heterogeneous nature of relational attributes.
    \item \emph{Structural attacks}, which target the relational entity
    graph itself, through rewiring operations that remain valid, i.e.,
    that preserve the integrity constraints of the database.
\end{itemize}
We focus on the latter setting, and investigate whether gradient-based
adversarial attacks can identify structural perturbations that degrade the
predictions of relational deep learning systems.

Concretely, we perturb PK--FK relationships while enforcing the integrity
constraints of relational theory, namely uniqueness, inclusion, and
functional dependencies. Each foreign key must reference a valid primary
key and remain attached to exactly one parent tuple, and functional
dependencies must hold throughout the perturbation. Because the graph
topology is entirely induced by these dependencies, perturbing the
database structure is equivalent to rewiring the induced graph; at the
database level, such a rewiring reassigns a tuple to an incorrect but
admissible foreign key, propagating semantic inconsistencies through the
PK--FK structure.

\usetikzlibrary{arrows.meta, positioning, fit, calc, decorations.pathreplacing}

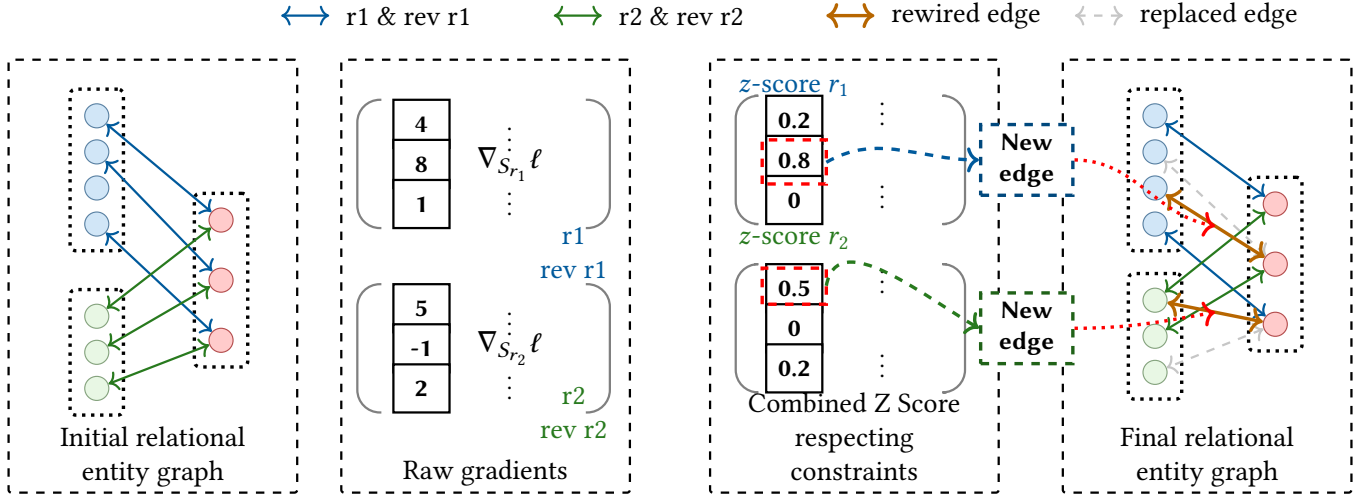
\begin{figure*}[t]
\centering
\resizebox{\textwidth}{!}{%
\begin{tikzpicture}[
    font=\sffamily\bfseries,
    node/.style={circle, draw=#1!70!black, fill=#1!35, minimum size=3mm, inner sep=0pt},
    panel/.style={draw, dashed, rounded corners=1pt, thick, minimum width=3.6cm, minimum height=5.4cm},
    group/.style={draw, dotted, very thick, rounded corners=2pt, inner sep=5pt},
    edge/.style={<->, thick},
    ghost/.style={<->, thick, gray!45, dashed},
    redarrow/.style={<->, red, dashed, very thick},
    mat/.style={draw, thick, fill=white, minimum width=7mm, minimum height=6mm},
    newbox/.style={draw=#1!80!black, dashed, very thick, fill=white, minimum width=1.55cm, minimum height=8mm, align=center}
]

\definecolor{blueRel}{HTML}{005A9C}
\definecolor{greenRel}{HTML}{287A1E}
\definecolor{tealRel}{HTML}{0E5A66}
\definecolor{orangeRel}{HTML}{B76600}
\definecolor{redRel}{HTML}{FF6B6B}
\definecolor{softBlue}{HTML}{7DB7E8}
\definecolor{softGreen}{HTML}{BDE7B3}

% Legend (top strip, clear of all edges)
\begin{scope}[shift={(1.6,3.25)}]
  \draw[edge, blueRel] (0,0) -- (0.6,0);
  \node[anchor=west, font=\large] at (0.7,0) {r1 \& rev r1};
  \draw[edge, greenRel] (3.4,0) -- (4.0,0);
  \node[anchor=west, font=\large] at (4.1,0) {r2 \& rev r2};
  \draw[edge, orangeRel, very thick] (6.8,0) -- (7.4,0);
  \node[anchor=west, font=\large] at (7.5,0) {rewired edge};
  \draw[ghost] (9.9,0) -- (10.5,0);
  \node[anchor=west, font=\large] at (10.6,0) {replaced edge};
\end{scope}

% Panels
\node[panel] at (0,0) {};
\node[panel] at (4.15,0) {};
\node[panel] at (8.75,0) {};
\node[panel] at (13.15,0) {};

% =======================
% Panel 1 : Initial (licit) graph
% =======================
\begin{scope}[xshift=0.4cm, yshift=0.7cm]
\foreach \i in {1,...,4}
  \node[node=softBlue] (a\i) at (-1.1,1.75-0.45*\i) {};
\foreach \i in {1,...,3}
  \node[node=softGreen] (b\i) at (-1.1,-0.75-0.45*\i) {};
\foreach \i in {1,...,3}
  \node[node=redRel] (c\i) at (0.45,0.75-0.75*\i) {};
\node[group, fit=(a1)(a4)] {};
\node[group, fit=(b1)(b3)] {};
\node[group, fit=(c1)(c3)] {};
\draw[edge, blueRel] (c1) -- (a1);
\draw[edge, blueRel] (c2) -- (a2);
\draw[edge, blueRel] (c3) -- (a4);
\draw[edge, greenRel] (c1) -- (b1);
\draw[edge, greenRel] (c2) -- (b2);
\draw[edge, greenRel] (c3) -- (b3);
\end{scope}

\node[align=center, font=\large] at (0,-2.25)
{Initial relational\\entity graph};

% =======================
% Panel 2 : Raw gradients
% =======================
\begin{scope}[yshift=0.45cm]
\foreach \v/\y in {4/1.45,8/0.95,1/0.45}
  \node[mat] at (3.35,\y) {\v};
\node at (4.45,1.35) {$\vdots$};
\node[font=\Large] at (4.45,0.95)
{$\nabla_{S_{r_1}}\ell$};
\node at (4.45,0.55) {$\vdots$};
\node[blueRel, align=center, font=\large]
at (5.25,-0.15) {r1\\rev r1};
\draw[gray, thick, rounded corners=8pt]
(2.85,1.75) -- (2.55,1.75) -- (2.55,0.15) -- (2.85,0.15);
\draw[gray, thick, rounded corners=8pt]
(5.40,1.75) -- (5.70,1.75) -- (5.70,0.15) -- (5.40,0.15);
\foreach \v/\y in {5/-0.85,-1/-1.35,2/-1.85}
  \node[mat] at (3.35,\y) {\v};
\node at (4.45,-0.95) {$\vdots$};
\node[font=\Large] at (4.45,-1.35)
{$\nabla_{S_{r_2}}\ell$};
\node at (4.45,-1.75) {$\vdots$};
\node[greenRel, align=center, font=\large]
at (5.25,-2.15) {r2\\rev r2};
\draw[gray, thick, rounded corners=8pt]
(2.85,-0.55) -- (2.55,-0.55) -- (2.55,-2.15) -- (2.85,-2.15);
\draw[gray, thick, rounded corners=8pt]
(5.40,-0.55) -- (5.70,-0.55) -- (5.70,-2.15) -- (5.40,-2.15);
\end{scope}

\node[align=center, font=\large] at (4.15,-2.45)
{Raw gradients};

% =======================
% Panel 3 : Combined z-score, candidate selection
% =======================
\begin{scope}[xshift=0.35cm,yshift=0.60cm]
\foreach \v/\y in {0.2/1.35,0.8/0.85,0/0.35}
  \node[mat] (zA\v) at (7.65,\y) {\v};
\draw[red, dashed, very thick]
(7.25,0.55) rectangle (8.05,1.10);
\draw[gray, thick, rounded corners=8pt]
(7.15,1.65) -- (6.92,1.65) -- (6.92,0.05) -- (7.15,0.05);
\draw[gray, thick, rounded corners=8pt]
(9.55,1.65) -- (9.78,1.65) -- (9.78,0.05) -- (9.55,0.05);
\node at (8.75,1.52) {$\vdots$};
\node at (8.75,0.48) {$\vdots$};
\node[blueRel, font=\large] at (7.65,1.75)
{$z$-score $r_1$};
\node[
    newbox=blueRel,
    minimum width=1.15cm,
    minimum height=8mm
] (ne1) at (10.55,0.85)
{New\\edge};
\draw[->, very thick, blueRel, dashed]
(8.05,0.82)
to[out=30,in=180]
(ne1.west);
\foreach \v/\y in {0.5/-0.75,0/-1.25,0.2/-1.75}
  \node[mat] at (7.65,\y) {\v};
\draw[red, dashed, very thick]
(7.25,-0.95) rectangle (8.05,-0.50);
\draw[gray, thick, rounded corners=8pt]
(7.15,-0.45) -- (6.92,-0.45) -- (6.92,-2.05) -- (7.15,-2.05);
\draw[gray, thick, rounded corners=8pt]
(9.55,-0.45) -- (9.78,-0.45) -- (9.78,-2.05) -- (9.55,-2.05);
\node at (8.75,-0.58) {$\vdots$};
\node at (8.75,-1.62) {$\vdots$};
\node[greenRel, font=\large] at (7.65,-0.15)
{$z$-score $r_2$};
\node[
    newbox=greenRel,
    minimum width=1.15cm,
    minimum height=8mm
] (ne2) at (10.55,-1.25)
{New\\edge};
\draw[->, very thick, greenRel, dashed]
(8.05,-0.72)
to[out=80,in=150]
(ne2.west);
\end{scope}

\node[align=center, font=\large] at (8.75,-2.00)
{Combined Z Score\\respecting\\constraints};

% =======================
% Panel 4 : Final graph (initial + 2 rewirings, still licit)
% =======================
\begin{scope}[yshift=0.90cm]
\foreach \i in {1,...,4}
  \node[node=softBlue] (fa\i) at (12.5,1.55-0.45*\i) {};
\foreach \i in {1,...,3}
  \node[node=softGreen] (fb\i) at (12.5,-0.75-0.45*\i) {};
\foreach \i in {1,...,3}
  \node[node=redRel] (fc\i) at (14.0,0.75-0.75*\i) {};
\node[group, fit=(fa1)(fa4)] {};
\node[group, fit=(fb1)(fb3)] {};
\node[group, fit=(fc1)(fc3)] {};
\draw[ghost] (fc2) -- (fa2);
\draw[ghost] (fc3) -- (fb3);
\draw[edge, blueRel] (fc1) -- (fa1);
\draw[edge, blueRel] (fc3) -- (fa4);
\draw[edge, greenRel] (fc1) -- (fb1);
\draw[edge, greenRel] (fc2) -- (fb2);
\draw[edge, orangeRel, very thick] (fc2) -- (fa3);
\draw[edge, orangeRel, very thick] (fc3) -- (fb1);
\end{scope}

\draw[->, red, dotted, very thick]
(ne1.east)
to[out=0,in=170]
($(fc2)!0.5!(fa3)$);
\draw[->, red, dotted, very thick]
(ne2.east)
to[out=0,in=190]
($(fc3)!0.5!(fb1)$);

\node[align=center, font=\large] at (13.15,-2.25)
{Final relational\\entity graph};

\end{tikzpicture}%
}
\caption{Gradient-guided constrained rewiring of a relational entity graph.
For each relation $r$ we compute the gradient $\nabla_{S_r}\ell$ of the batch
loss with respect to its adjacency and combine the forward and reverse
contributions into a single candidate score. To make sensitivities comparable
across relations, these scores are normalized per relation; the $z$-score is
shown here as a representative scheme, with the full set of normalizations we
consider summarized in Table~\ref{tab:score_normalization}. The highest-scoring admissible
candidates---those respecting the schema's integrity constraints---are selected
as new edges and applied to the graph. Each selected rewiring replaces a single
foreign-key edge by an admissible one (orange), leaving the degree-one
constraint intact; the replaced edges are shown faded.}
\label{fig:constrained_rewiring}
\end{figure*}

The resulting search problem is combinatorial. Each admissible edit must
simultaneously satisfy the integrity constraints and respect a global
perturbation budget, and the effect of edits is highly non-additive
because message passing in the GNN couples them. Exhaustively evaluating
all admissible rewiring combinations is therefore intractable, and the
space grows with the number of eligible foreign keys and admissible
destinations, which motivates the use of approximate search heuristics.

\paragraph*{Contributions} We investigate seven such heuristic strategies: two random sampling baselines and five gradient-guided variants. The gradient-based variants are derived from a prior work \cite{zugner2018adversarial} which uses differentiable edge masks to estimate the sensitivity of the prediction loss to candidate rewirings and to guide selection. These designs follow approaches that have proven effective
across a variety of combinatorial optimization
problems~\cite{cai2018kbgan, zhang2019data, zhao2024untargeted,
chen2018fast, chen2020mga, ma2022adversarial, zugner2020adversarial,
guo2026graph, ma2020towards, dai2018adversarial}. Empirically,
gradient-based attacks consistently outperform random baselines on
regression tasks. On classification tasks the picture is more nuanced:
gains are smaller, which we attribute to low label-flip rates and the
greater local stability of classification outputs under sparse
perturbations.

\section{Related Work}
%\subsection{Adversarial Learning and Relational Deep Learning}

Adversarial learning has been studied across a broad range of data
modalities, progressively moving from i.i.d.\ data to structured and
relational settings.

\paragraph*{Euclidean Data}
Early works cast \emph{data poisoning} as a bilevel optimization problem,
enabling gradient-based and targeted manipulations of the training
data~\cite{biggio2012poisoning, jagielski2018manipulating}.
Complementarily, \cite{steinhardt2017certified} provide theoretical
guarantees on worst-case robustness under bounded perturbations.

\paragraph*{Graph-Structured Data}
Adversarial attacks on graphs exploit structural dependencies rather than
feature noise alone. Gradient-based methods perturb the adjacency matrix
to degrade model performance~\cite{zugner2018adversarial, chen2018fast,
chen2020mga}, while more realistic settings consider constrained or
black-box attacks, often framed as influence
maximization~\cite{ma2020towards, ma2022adversarial}. A complementary line
of work studies how message passing propagates such perturbations and how
to defend against it~\cite{guo2026graph}.

\paragraph*{Knowledge Graphs and Tabular Data}
In structured domains, attacks adapt to discrete and heterogeneous
representations. For knowledge graphs, adversarial methods target triples
through poisoning or adversarial training~\cite{cai2018kbgan,
zhang2019data, zhao2024untargeted}. In tabular settings, attacks focus on
label corruption and backdoor mechanisms under feature
constraints~\cite{chang2023fast, tajalli2025catback}, with theoretical
foundations in malicious noise models~\cite{kearns1988learning}.

\paragraph*{Relational Deep Learning}
Recent work introduces \emph{Relational Deep Learning} (RDL), which models
multi-table databases as heterogeneous temporal graphs and enables
end-to-end learning via GNNs~\cite{fey2024position}. Benchmarks such as
RelBench show that this paradigm outperforms traditional feature
engineering pipelines~\cite{robinson2024relbench, gu2026relbench},
supported by modular frameworks for heterogeneous data
encoding~\cite{hu2024pytorch}. A growing body of work highlights its open
challenges, including scalability, temporality, and
heterogeneity~\cite{dwivedi2025relational}.

\paragraph*{Learning over Relational Databases}
The database community has long studied machine learning directly over
relational data, exploiting the schema to make learning efficient.
Factorized and in-database learning train models over the join of multiple
tables without materializing it, and notably exploit functional dependencies
to reduce model dimensionality~\cite{schleich2016learning,
abokhamis2018indatabase, olteanu2016factorized}; relational deep learning can
be viewed as the GNN-based successor to this line, replacing closed-form
aggregates with message passing. Integrity constraints---functional and
inclusion dependencies, referential integrity---are foundational to
relational theory~\cite{abiteboul1995foundations}, and a large body of work
relies on them for data cleaning, where constraint-based repair detects and
fixes violations to restore a consistent
database~\cite{rekatsinas2017holoclean}. Our attack is the adversarial
counterpart of this process: rather than repairing inconsistencies, it
searches for integrity-preserving edits that keep the database consistent yet
degrade the downstream model.

\paragraph*{Gap}
Adversarial attacks are by now well understood in the Euclidean, graph,
and tabular settings, but their extension to relational deep learning
remains largely unexplored. Crucially, the techniques developed for
generic graphs do not transfer directly: they assume that edges can be
inserted or deleted freely, whereas in a relational database the topology
is induced by the schema and every edit must preserve its integrity
constraints. The admissible perturbations are therefore not arbitrary
adjacency changes but integrity-preserving rewirings of foreign-key
references, which is the setting we study in this work.

\section{Preliminaries}

\paragraph*{Representing Relations as Heterogeneous Entity Graphs}

We follow the RelBench construction, which represents a relational database
as a \emph{heterogeneous entity graph}: each row of each table becomes a
node, and each foreign-key dependency induces a typed edge (see
Figure~\ref{fig:db_to_poisoned_prediction}).

The node set is not homogeneous. It is partitioned by table type,
$\mathcal V=\bigsqcup_{\tau\in\mathcal T}\mathcal V_\tau$, where $\mathcal T$
is the set of table (node) types, $\mathcal V_\tau$ is the set of tuples
(rows) of table $\tau$, and $\bigsqcup$ denotes a disjoint union. All nodes
thus live in a common vertex set $\mathcal V$ but remain separated into
semantically distinct subsets according to their table type.

We define the heterogeneous graph as
$\mathcal G=(\mathcal V,\{\mathcal E_r\}_{r\in\mathcal R})$, where each
relation type $r\in\mathcal R$ carries its own edge set
$\mathcal E_r\subseteq \mathcal V_{\mathrm{src}(r)}\times
\mathcal V_{\mathrm{dst}(r)}$. For a relation $r$, the types
$\mathrm{src}(r)\in\mathcal T$ and $\mathrm{dst}(r)\in\mathcal T$ denote the
table at which the edge starts and the table at which it ends,
respectively, so that $(i,j)\in\mathcal E_r$ means that tuple $i$ of type
$\mathrm{src}(r)$ is linked to tuple $j$ of type $\mathrm{dst}(r)$.

The connectivity of each relation $r$ is encoded by an adjacency matrix
$S_r\in\{0,1\}^{n_{\mathrm{src}(r)}\times n_{\mathrm{dst}(r)}}$, with
$n_{\mathrm{src}(r)}=|\mathcal V_{\mathrm{src}(r)}|$ and
$n_{\mathrm{dst}(r)}=|\mathcal V_{\mathrm{dst}(r)}|$, whose entries satisfy
$(S_r)_{ij}=1\iff(i,j)\in\mathcal E_r$.

A relational dependency is naturally directed: a foreign key points toward
the referenced primary key. In practice (in RelBench and PyG
implementations) each such dependency is duplicated into two edge types, so
that message passing can propagate in both directions.

\paragraph*{Node Features}
Each tuple carries tabular attributes (categorical, numerical, text,
timestamp, etc.), which are independently embedded and merged into a single
dense per-tuple representation following the tensor-frame paradigm:
$X=\{X_\tau\}_{\tau\in\mathcal T}$, with
$X_\tau\in\mathbb R^{|\mathcal V_\tau|\times d_\tau}$. A heterogeneous GNN
$f_\theta$ then predicts from both the features and the typed relations,
$\hat y=f_\theta(X,\{S_r\}_{r\in\mathcal R})$.
% TODO: add citation for heterogeneous GNNs / GraphSAGE

\section{Constrained Adversarial Attacks on Relational Graphs}

\subsection{Preserving Foreign-Key Integrity}

\paragraph{Relational graph semantics and perturbation constraints.}
Foreign-key to primary-key (FK--PK) dependencies form the structural
backbone of the relational graph: they connect tuples from different tables
wherever the schema defines a referential constraint. The graph topology is
therefore induced directly by the referential structure of the database.

Each foreign-key tuple references exactly one parent tuple. Admissible
attacks consequently cannot insert or delete edges freely; perturbations are
restricted to \emph{rewiring operations}, in which the parent referenced by
a foreign key is replaced by another admissible primary-key node (and the
symmetric reverse edge is updated accordingly), preserving the degree-one
constraint. Formally, a perturbation replaces an edge
$(v_{\mathrm{FK}}, v_{\mathrm{PK}}^{\mathrm{old}})$ by
$(v_{\mathrm{FK}}, v_{\mathrm{PK}}^{\mathrm{new}})$, so at most one rewiring
can be applied to a given FK node per attack step.

\paragraph{Budget.}
Attacks are performed under a global perturbation budget $B$, the maximum
number of rewiring operations allowed during a single attack. The budget is
distributed across the relations of the heterogeneous graph, so that
perturbations may affect several edge types while remaining sparse and
semantically valid.

\paragraph{Preserving functional dependencies.}
To keep the database consistent under perturbation, we distinguish between
\emph{mutable} and \emph{immutable} dependencies. A dependency is said to be
\emph{mutable} if its value can be modified by the attack while preserving
database integrity. Mutable dependencies are further divided into two classes:
\emph{coupled rewiring} and \emph{local rewiring}. In contrast,
\emph{immutable dependencies} cannot be modified because they encode entity
identities or protected structural constraints. We describe them next.

%To keep the database consistent under perturbation, we distinguish three
%classes of admissible rewiring operations: \emph{local rewiring},
%\emph{coupled rewiring}, and \emph{immutable dependencies}.

\textbf{Local rewiring} acts independently on a single FK--PK dependency: a
foreign key is reassigned to another admissible parent while preserving
referential integrity and all degree constraints imposed by the schema.

\textbf{Coupled rewiring} concerns groups of foreign keys whose semantics are
jointly constrained, where editing one dependency in isolation could yield an
incoherent configuration. Such foreign keys must be perturbed simultaneously,
as a single atomic operation.

\textbf{Immutable} dependencies are globally sensitive or structurally
critical relations whose modification could propagate large-scale
inconsistencies through the database; these are excluded from the attack
space. 

\paragraph{Overview of the threat model.} In short, the proposed threat model seeks adversarial rewirings while preserving the integrity of the relational database. No tuples are inserted or removed; instead, the attack is restricted to modifying foreign-key references. These references are rewired such that their new values remain valid with respect to the corresponding attribute domains and integrity constraints. As a result, the size of the database remains unchanged throughout the attack.

\subsection{Why We Need Heuristics}

We first ask what it would cost to find the best possible attack. Let $C$ be
the total number of legal rewiring candidates across all relations, assuming
the full set of admissible FK--PK perturbations is known in advance. Under a
budget $B$, the attacker must select $B$ of these candidates, and even in
this simplified view the number of possible attacks already grows
combinatorially as $\binom{C}{B}$.

This estimate still understates the true difficulty, because rewiring
operations are not independent. Several relations encode the same information
through forward and reverse edges, so a perturbation on one relation may
constrain or invalidate perturbations on another; moreover, structurally
coupled foreign keys must be modified jointly to preserve integrity. These
dependencies add further combinatorial constraints, making the search for an
optimal perturbation considerably harder than a plain subset-selection
problem and motivating the use of approximate heuristics.

\section{White-Box Gradient-Based Attack} \label{sec:method}

\subsection{Attack Objective}

We assume a trained heterogeneous GNN
$\hat y = f_\theta(X,\{S_r\}_{r\in\mathcal R})$, in which the node features
$X$ are fixed and only the relational structures $\{S_r\}_{r\in\mathcal R}$
may be modified. The attacker seeks perturbed relation matrices
$\{\widetilde S_r\}_{r\in\mathcal R}$ that worsen the prediction while
respecting the edit budget and the integrity constraints introduced above:
\[
\max_{\{\widetilde S_r\}}
\;
\ell_{\mathcal B}
\!\left(
f_\theta(X,\{\widetilde S_r\})_{\mathcal B},
y_{\mathcal B}
\right),
\]
where $\mathcal B$ is the mini-batch under consideration and
$\ell_{\mathcal B}$ is the loss evaluated on its labels $y_{\mathcal B}$.
Unlike targeted attacks, no particular subset of nodes is singled out: the
objective is to apply as few valid structural edits as possible while
maximally increasing the loss over the selected entities.

\subsection{Attackable Model}

\paragraph{The model.}
The original GNN is not directly amenable to structural attacks: its edge
index matrices carry no gradient, since they are purely combinatorial
representations of the graph structure. To recover a gradient signal over the
relations, we define an alternative \emph{attackable} model. The parameters
of the pretrained GNN are transferred into this model and frozen, and
differentiable edge weights are introduced for each relation, allowing
gradients to flow through the graph structure during backpropagation.

\paragraph{Masks and sparsity.}
Because the pipeline is implemented in PyTorch Geometric, graph structures are
stored as sparse edge-index tensors, so gradients can only propagate through
edges that already exist. To expose unseen rewiring candidates to the
gradient, we augment the graph with a set of admissible candidate edges
sampled under the semantic and structural constraints defined above:
existing edges receive weight $1$, candidate edges receive an initial weight
of $0$, and a differentiable mask over these weights lets gradients propagate
through both the original topology and the candidate edges.

\subsection{Gradient-Based Rewiring}

\paragraph{First-order candidate scoring.}
The mask gives us, for each relation, a first-order signal: entries with
large gradient magnitude are the structural edits to which the loss is most
sensitive, which is exactly the direction a first-order (gradient-sign)
attack would follow. We use this signal only to \emph{rank} candidate
rewirings; the formal first-order derivation is given in
Appendix~\ref{proof:FGSM}. Concretely, let
\[
G_r = \nabla_{S_r}\,\ell_{\mathrm{batch}}(S_r)
\]
denote the gradient of the batch loss with respect to the direct relation
$r$, and let
\[
G_{r^{-1}} = \nabla_{S_{r^{-1}}}\,\ell_{\mathrm{batch}}(S_{r^{-1}})
\]
be the gradient of its reverse relation. Since each rewiring simultaneously
modifies the forward FK$\rightarrow$PK edge and its reverse edge, candidate
sensitivity is measured by the combined score
\[
(G_r)_{f,p} + (G_{r^{-1}})_{p,f}.
\]
For a source node $f$, large positive values of this score indicate
promising rewiring candidates.

%%%%%%%%%%%%%%%%%%%%%%%%%%%%%%
\begin{table*}[t]
\centering
\caption{
Gradient-based candidate scoring and normalization strategies.
For all methods, $G_r$ and $G_{r^{-1}}$ denote the gradients associated with
the forward and reverse relations, respectively, $\mu_r$ and $\sigma_r$
denote the mean and standard deviation of candidate scores for relation $r$,
$m_r$ is the median, $\mathrm{MAD}_r$ the median absolute deviation, and
$s_r^{\min}$ and $s_r^{\max}$ the minimum and maximum candidate scores
observed for relation $r$.
}
\label{tab:score_normalization}

\small
\begin{tabular}{
>{\centering\arraybackslash}m{2.3cm}
>{\centering\arraybackslash}m{6.2cm}
>{\centering\arraybackslash}m{5.5cm}
}
\toprule
\textbf{Method} & \textbf{Definition} & \textbf{Purpose} \\
\midrule

Raw Gradient
&
$\displaystyle
\mathrm{score}_{\mathrm{raw}}(f,p)
=
(G_r)_{f,p}
+
(G_{r^{-1}})_{p,f}
$
&
Direct first-order sensitivity estimate. Preserves absolute gradient magnitude.
\\[2pt]
\cmidrule(lr){1-3}

Z-score
&
$\displaystyle
\mathrm{score}_{z}(f,p)
=
\frac{\mathrm{score}_{\mathrm{raw}}(f,p)-\mu_r}
{\sigma_r+\varepsilon}
$
&
Removes relation-specific scale effects using the mean and standard deviation.
\\[2pt]
\cmidrule(lr){1-3}

Robust Z-score
&
$\displaystyle
\mathrm{score}_{rz}(f,p)
=
\frac{\mathrm{score}_{\mathrm{raw}}(f,p)-m_r}
{\mathrm{MAD}_r+\varepsilon}
$
&
More robust to outliers and heavy-tailed gradient distributions.
\\[2pt]
\cmidrule(lr){1-3}

Min-Max
&
$\displaystyle
\mathrm{score}_{mm}(f,p)
=
\frac{\mathrm{score}_{\mathrm{raw}}(f,p)-s_r^{\min}}
{s_r^{\max}-s_r^{\min}+\varepsilon}
$
&
Places all candidate scores on a common $[0,1]$ scale within each relation,
facilitating comparisons across relations with different gradient ranges.
\\

\bottomrule
\end{tabular}
\end{table*}

%%%%%%%%%%%%%%%%%%%%%%%%%%%%%

Exhaustively testing all destinations is too costly: with $N$ eligible
source nodes and $C$ candidates per node, the prototype requires
$\mathcal{O}(NC)$ forward passes, so we restrict the search to a small
candidate set. 

Moreover, raw gradient magnitudes can vary substantially across source nodes and relations, making direct comparison difficult. 
To reduce this bias, we consider several normalization strategies applied to the combined forward--reverse sensitivity score.
Their definitions and intended effects are summarized in Table~\ref{tab:score_normalization}.

% Moreover, raw gradient magnitudes vary widely across source
% nodes and relations, which biases a direct comparison. To make sensitivities
% comparable, we standardize the combined forward--reverse score with a
% per-relation $z$-score:
% \[
% \mathrm{score}_{\mathrm{raw}}(f,p)
% =
% (G_r)_{f,p} + (G_{r^{-1}})_{p,f},
% \]
% \[
% \operatorname{score}_z(f,p)
% =
% \frac{\mathrm{score}_{\mathrm{raw}}(f,p)-\mu_r}{\sigma_r+\varepsilon}.
% \]
% where $\mu_r$ and $\sigma_r$ are the mean and standard deviation of the
% candidate scores of relation $r$, and $\varepsilon$ is a small constant for
% numerical stability. This places rewiring sensitivities on a common scale,
% both within and across relations.

\paragraph{Candidate set.}
For each eligible source node $f$ with current parent $p_{\mathrm{old}}(f)$,
the set of admissible candidates for $r\in\mathcal R$ is
\[
\mathcal C_r(f)
=
\Bigl\{
p \in \mathcal V_{\mathrm{dst}(r)}
\;:\;
p \neq p_{\mathrm{old}}^{(r)}(f)
\Bigr\}.
\]
Given these candidates, we define two strategies for selecting both the edges
to perturb and their new destinations; in both, perturbations are optimized
jointly across all relations $r\in\mathcal R$.

\paragraph{Direct global gradient selection.}
The first strategy ranks all admissible candidates
$\mathcal C_r(f)$ ($\forall r\in\mathcal R$) by their first-order score
$\operatorname{score}(f,p)$ and keeps the top-$B$ globally:
\[
\mathcal A^\star
=
\operatorname{TopB}_{r\in\mathcal R,\;f,\;p\in\mathcal C_r(f)}
\operatorname{score}(f,p).
\]
This yields a fully gradient-driven attack with no explicit forward
re-evaluation.

\paragraph{Shortlist-based reranking.}
The second strategy adds an intermediate shortlist stage to reduce the number
of explicit forward evaluations. For each source node $f$ and relation $r$,
only the top-$K$ gradient candidates are kept, focusing the search on the most promising destinations:
\[
\mathcal C_{K}^{(r)}(f)
=
\operatorname{TopK}_{p\in\mathcal C_r(f)}
(G_r)_{f,p},
\qquad
K \ll |\mathcal C_r(f)|,
\]
For each candidate $p \in \mathcal C_{K}^{(r)}(f)$ we then evaluate the exact loss variation induced by the rewiring,
\[
\begin{aligned}
\Delta \ell_r(f,p)
={}&
\ell_{\mathrm{batch}}
\Bigl(
\{S_u\}_{u\in\mathcal R \setminus \{r,r^{-1}\}}
\cup
S_r^{(f\to p)}
\cup
S_{r^{-1}}^{(p\to f)}
\Bigr) \\
&-
\ell_{\mathrm{batch}}
\bigl(
\{S_u\}_{u\in\mathcal R}
\bigr),
\end{aligned}
\]
where
\[
\{S_u\}_{u\in\mathcal R \setminus \{r,r^{-1}\}}
\cup
S_r^{(f\to p)}
\cup
S_{r^{-1}}^{(p\to f)}
\]
denotes the adjacency matrices obtained after replacing the edge
$(f,p_{\mathrm{old}}^{(r)}(f))$ by $(f,p)$ in the direct relation $r$,
together with its symmetric counterpart
$(p_{\mathrm{old}}^{(r)}(f),f)\mapsto(p,f)$ in the reverse relation
$r^{-1}$. The best local rewiring for node $f$ is
\[
p^\star_r(f)
\in
\arg\max_{p\in\mathcal C_{K}^{(r)}(f)}
\Delta \ell_r(f,p),
\]
and the locally optimal candidates are aggregated across relations and
ranked by their achieved loss increase, keeping the top-$B$:
\[
\mathcal A^\star
=
\operatorname{TopB}_{r\in\mathcal R,\;f}
\Delta \ell_r\bigl(f,p_r^\star(f)\bigr).
\]
The final perturbation set thus satisfies the global budget while allowing
coordinated changes across multiple relations.

\paragraph{Note on the non-additivity of structural perturbations.}
The effect of structural perturbations is inherently non-additive: the impact
of a rewiring on the loss can depend strongly on which other perturbations are
applied simultaneously. Two edits that each increase the loss in isolation may,
when combined, partially cancel or even reverse their joint effect on the
prediction.
\section{Experiments}

\subsection{Experimental Setup}

All experiments were conducted on a CPU-only virtual machine with $96$ cores
and $396$\,GB of RAM. Our implementation builds on the Relational Deep
Learning (RDL) framework of RelBench, itself implemented on top of PyTorch
Geometric for heterogeneous graph learning. 
%All datasets and code are available at \url{https://relbench.stanford.edu/}.

\subsection{Dataset}

We deliberately selected a dataset with a rich, highly structured relational
schema rather than collections of a few isolated tables, so that we can
study how adversarial perturbations propagate across relations. The
attack evaluation reported below focuses on \texttt{rel-f1}, which contains multiple interconnected entity tables, interaction tables,
foreign keys in different tables pointing to the same primary key (denoted as \emph{coupled}), and hierarchical dependencies, making it 
representative of realistic relational systems. An overview of the dataset
and its prediction tasks is given in Table~\ref{tab:datasets_overview},
and the structural constraints and dependency-preserving policies enforced
during rewiring are summarized in Table~\ref{tab:all_schema_constraints}. 

% ---------------------------------------------------------------------------
% Table 1: datasets overview
% ---------------------------------------------------------------------------
\begin{table*}[t]
\centering
\caption{Overview of the RelBench \texttt{rel-f1} and prediction tasks considered in this work.}
\label{tab:datasets_overview}
\renewcommand{\arraystretch}{1.15}
\begin{tabular}{c c c c c c c}
\toprule
Dataset & \# Tables & \# Rows & \# Columns & \# Tasks & Task name & Task type \\
\midrule

\multirow{4}{*}{\texttt{rel-f1}}
& \multirow{4}{*}{9}
& \multirow{4}{*}{74,063}
& \multirow{4}{*}{67}
& \multirow{4}{*}{4}
& \texttt{driver-position} & Regression \\

&&&&
& \texttt{qualifying-position} & Regression \\

&&&&
& \texttt{driver-dnf} & Binary classification \\

&&&&
& \texttt{driver-top3} & Binary classification \\

% \midrule

% \multirow{5}{*}{\texttt{rel-ratebeer}}
% & \multirow{5}{*}{13}
% & \multirow{5}{*}{13,787,005}
% & \multirow{5}{*}{221}
% & \multirow{5}{*}{5}
% & \texttt{beer-churn} & Binary classification \\

% &&&&
% & \texttt{user-churn} & Binary classification \\

% &&&&
% & \texttt{brewer-dormant} & Binary classification \\

% &&&&
% & \texttt{user-count} & Regression \\

% &&&&
% & \texttt{beer\_ratings-total\_score} & Regression \\

% \midrule

% \multirow{4}{*}{\texttt{rel-stack}}
% & \multirow{4}{*}{7}
% & \multirow{4}{*}{4,247,264}
% & \multirow{4}{*}{52}
% & \multirow{4}{*}{4}
% & \texttt{user-engagement} & Binary classification \\

% &&&&
% & \texttt{user-badge} & Binary classification \\

% &&&&
% & \texttt{badges-class} & Multiclass classification \\

% &&&&
% & \texttt{post-votes} & Regression \\

\bottomrule
\end{tabular}
\end{table*}

% ---------------------------------------------------------------------------
% Table 5: schema constraints and attack policies
% ---------------------------------------------------------------------------
\begin{table*}[t]
\centering
\caption{Schema constraints and attack policies for the \texttt{rel-f1} dataset.}
\label{tab:all_schema_constraints}
\small
\begin{tabular}{p{5.2cm}p{3.6cm}p{4.6cm}p{1.4cm}}
\toprule
Relation / Foreign Key & Description & Constraint type & Mutable \\
\midrule

\multicolumn{4}{c}{\textbf{Formula 1 (\texttt{rel-f1})}} \\
\midrule

\texttt{results.raceId}
& Race participation result
& Coupled with \texttt{driverId}, \texttt{constructorId}
& Yes \\

\texttt{results.driverId}
& Driver participation result
& Coupled with \texttt{raceId}, \texttt{constructorId}
& Yes \\

\texttt{results.constructorId}
& Constructor participation result
& Coupled with \texttt{raceId}, \texttt{driverId}
& Yes \\

\texttt{qualifying.raceId}
& Qualifying session assignment
& Coupled with \texttt{driverId}, \texttt{constructorId}
& Yes \\

\texttt{qualifying.driverId}
& Driver qualifying participation
& Coupled with \texttt{raceId}, \texttt{constructorId}
& Yes \\

\texttt{qualifying.constructorId}
& Constructor qualifying participation
& Coupled with \texttt{raceId}, \texttt{driverId}
& Yes \\

\texttt{standings.raceId}
& Driver standings snapshot
& Coupled with \texttt{driverId}
& Yes \\

\texttt{standings.driverId}
& Driver standings snapshot
& Coupled with \texttt{raceId}
& Yes \\

\texttt{constructor\_standings.raceId}
& Constructor standings snapshot
& Coupled with \texttt{constructorId}
& Yes \\

\texttt{constructor\_standings.constructorId}
& Constructor standings snapshot
& Coupled with \texttt{raceId}
& Yes \\

\texttt{constructor\_results.raceId}
& Constructor race result
& Coupled with \texttt{constructorId}
& Yes \\

\texttt{constructor\_results.constructorId}
& Constructor race result
& Coupled with \texttt{raceId}
& Yes \\

\texttt{races.circuitId}
& Race $\rightarrow$ circuit assignment
& Structural FK (sensitive)
& Restricted \\

\texttt{drivers.driverId}, \texttt{constructors.constructorId},
\texttt{circuits.circuitId}
& Core entity identifiers
& Immutable PKs
& No \\

\bottomrule
\end{tabular}
\end{table*}

\subsection{Protocol}

\paragraph{Model architecture.}
We use a \emph{heterogeneous GraphSAGE} architecture with two
message-passing layers. A shallow architecture is important in our setting:
it limits both \emph{vanishing gradients} (which would weaken the attack
signal) and \emph{oversmoothing}. All hidden representations use a channel 
dimension of $128$ or $16$, depending on the task (see Table~\ref{tab:hyperparams} in the Appendix). For regression tasks the prediction head is left
\emph{linear}, while for classification we apply a \emph{softmax} over the
output logits. The attacked objective is task-dependent: \emph{Mean Absolute
Error} (MAE) for regression, and \emph{Binary} or standard \emph{Cross-Entropy}
for classification, depending on the task.

\paragraph{Mini-batch sampling.}
Each mini-batch is generated by \emph{temporal neighborhood sampling} with at
most $128$ neighbors per hop and a batch size of $512$. All reported attacks
operate on a fixed mini-batch sampled from \texttt{rel-f1}; the trained model
is frozen, so no parameter is updated during the attack and only the graph
structure of the batch is modified.

\paragraph{Perturbation budget.}
We focus on highly localized attacks under small budgets, i.e., adversarial
rewiring that modifies between $1$ and $100$ edges.

\paragraph{Attack strategies.}
We compare seven attack strategies that are extensions of the ones defined in Section~ \ref{sec:method}.
\begin{enumerate}

\item \textbf{Random.} Admissible perturbations are sampled uniformly at random among all valid
rewirings, while respecting the relational schema constraints. 

\item \textbf{Random\,+\,Exact.} A set of admissible perturbations is first sampled uniformly at random.
The sampled candidates are subsequently re-ranked through exact forward
evaluation on the attacked model, and the perturbation producing the
largest degradation is selected. 

\item \textbf{Gradient.}
All admissible perturbations are scored using first-order sensitivity
information obtained through backpropagation. Candidates are ranked
according to their raw gradient-based scores and selected directly without
additional evaluation. 

\item \textbf{Gradient\,+\,Exact.}
Gradient scores are first used to rank all admissible perturbations.
The resulting candidates are then re-evaluated through exact forward
passes on the attacked model, allowing us to assess how accurately the
first-order approximation predicts the true attack effect. 

\item \textbf{Gradient\,+\,Z-score.}
Instead of using raw gradient scores, candidate perturbations are
normalized using a per-relation Z-score transformation. This reduces the
impact of scale differences across relations and encourages a more diverse
selection of perturbations before applying the attack. 

\item \textbf{Gradient\,+\,Robust Z-score.}
Candidates are normalized using a robust Z-score based on the median and
median absolute deviation (MAD). This strategy is less sensitive to
heavy-tailed gradient distributions and outliers, providing a more stable
ranking of perturbations across heterogeneous relations.

\item \textbf{Gradient\,+\,Min-Max.}
Candidates are normalized independently within each relation using min--max
scaling, mapping the lowest gradient score to $0$ and the highest to $1$.
The objective is to reduce the dominance of relations that naturally produce
larger gradient magnitudes and to make perturbation scores comparable across
relations.
\end{enumerate}

% \paragraph{Attack strategies.}
% We compare four strategies.
% \textbf{(1) Gradient-based.} Candidate perturbations are ranked directly by
% the first-order sensitivity scores obtained through backpropagation. Because
% the sampled mini-batches are small, we can evaluate \emph{all} admissible
% candidates during gradient computation without aggressive pruning.
% \textbf{(2) Gradient\,+\,shortlist.} Gradient scoring is followed by a
% \emph{shortlist} stage: only the top-$200$ candidates are kept before an exact
% forward re-evaluation, letting us test whether restricting the search to the
% most promising gradient directions improves the attack.
% \textbf{(3) Constrained random.} Perturbations are sampled uniformly at random
% among admissible rewirings, respecting all schema constraints (foreign-key
% validity, immutable primary keys, coupled dependencies, and relation
% compatibility).
% \textbf{(4) Random\,+\,shortlist.} A pool of $200$ random admissible
% perturbations is generated and then exactly re-evaluated, isolating the
% contribution of the gradient signal from the mere effect of evaluating a
% restricted candidate subset.

\begin{figure}
    \centering
    \includegraphics[width=1\linewidth]{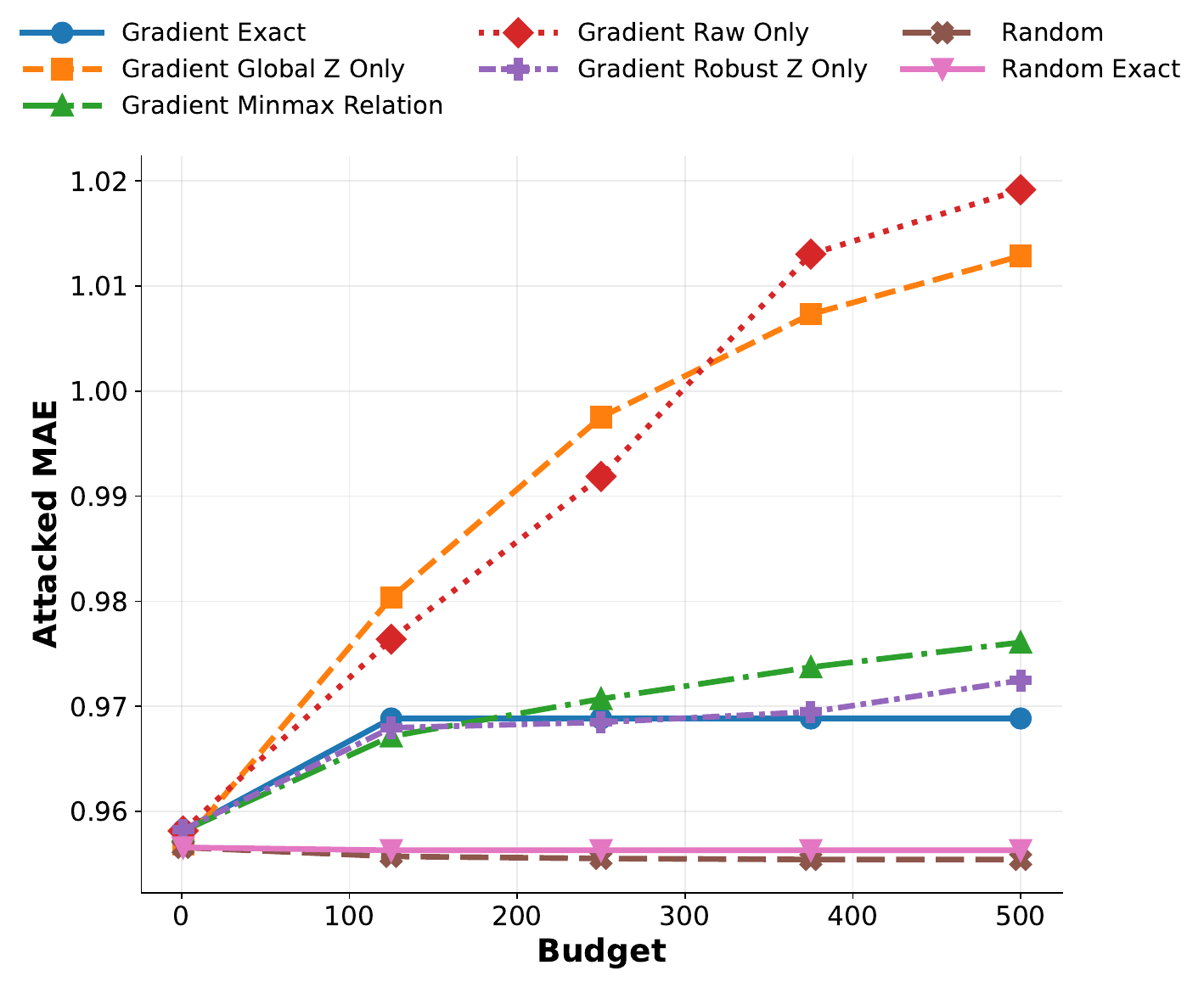}
    \caption{Attacked MAE on the \texttt{qualifying-position} task of the \texttt{rel-f1} dataset as a function of the attack budget (1--500 edge perturbations). Higher values correspond to stronger attacks. Results are reported for random seed 41.}
    \label{fig:non_additivity}
\end{figure}

% ---------------------------------------------------------------------------
% Main results table
% ---------------------------------------------------------------------------
\begin{table*}[t]
\centering
\caption{
Performance under structural attacks on the \texttt{rel-f1} dataset.
For classification tasks (\texttt{driver-dnf}, \texttt{driver-top3}), results are reported as attacked accuracy (\%) and lower values indicate stronger attacks.
For regression tasks (\texttt{driver-position}, \texttt{qualifying-position}), results are reported as attacked MAE and higher values indicate stronger attacks.
Best method for each budget is shown in bold, while the second best is underlined. Results are averaged over five independent random seeds. %that are reported in the Appendix (Table~\ref{app:hyperparameters}).
}
\label{tab:relf1_all_results}
\resizebox{\textwidth}{!}{%
\begin{tabular}{llccccc}
\toprule
\textbf{Task} & \textbf{Method}
& \textbf{$B=1$}
& \textbf{$B=25$}
& \textbf{$B=50$}
& \textbf{$B=75$}
& \textbf{$B=100$} \\
\midrule

\multirow{6}{*}{\texttt{driver-dnf}}
& Random
& \underline{77.37 $\pm$ 0.58}
& 77.37 $\pm$ 0.58
& 77.37 $\pm$ 0.58
& 77.37 $\pm$ 0.58
& 77.37 $\pm$ 0.58 \\

& Random + exact rerank
& \textbf{76.34 $\pm$ 0.59}
& \textbf{76.24 $\pm$ 0.45}
& \textbf{76.24 $\pm$ 0.45}
& \underline{76.24 $\pm$ 0.45}
& \underline{76.24 $\pm$ 0.45} \\

& Gradient raw
& \underline{77.37 $\pm$ 0.58}
& 77.37 $\pm$ 0.58
& 77.37 $\pm$ 0.58
& 76.54 $\pm$ 0.86
& 76.54 $\pm$ 0.86 \\

& Gradient z-score
& \underline{77.37 $\pm$ 0.58}
& 77.37 $\pm$ 0.58
& \underline{76.44 $\pm$ 0.73}
& \textbf{75.51 $\pm$ 2.04}
& \textbf{75.51 $\pm$ 2.04} \\

& Gradient robust z-score
& \underline{77.37 $\pm$ 0.58}
& \underline{76.44 $\pm$ 0.73}
& \underline{76.44 $\pm$ 0.73}
& 76.54 $\pm$ 0.86
& 76.54 $\pm$ 0.86 \\

& Gradient + exact rerank
& \underline{77.37 $\pm$ 0.58}
& 77.37 $\pm$ 0.58
& \underline{76.44 $\pm$ 0.73}
& \textbf{75.51 $\pm$ 2.04}
& \textbf{75.51 $\pm$ 2.04} \\

\midrule

\multirow{6}{*}{\texttt{driver-top3}}
& Random
& 77.63 $\pm$ 3.72
& \underline{77.63 $\pm$ 3.72}
& \underline{77.53 $\pm$ 3.86}
& \underline{77.53 $\pm$ 3.86}
& \underline{77.53 $\pm$ 3.86} \\

& Random + exact rerank
& 77.63 $\pm$ 3.72
& \underline{77.63 $\pm$ 3.72}
& \textbf{76.97 $\pm$ 2.79}
& \textbf{76.97 $\pm$ 2.79}
& \underline{76.88 $\pm$ 2.90} \\

& Gradient raw
& 77.63 $\pm$ 3.72
& \textbf{76.97 $\pm$ 2.79}
& \textbf{76.97 $\pm$ 2.79}
& \textbf{76.97 $\pm$ 2.79}
& \textbf{76.32 $\pm$ 3.28} \\

& Gradient z-score
& 77.63 $\pm$ 3.72
& \textbf{76.97 $\pm$ 2.79}
& \textbf{76.97 $\pm$ 2.79}
& \textbf{76.97 $\pm$ 2.79}
& \underline{76.97 $\pm$ 2.79} \\

& Gradient robust z-score
& 77.63 $\pm$ 3.72
& \textbf{76.97 $\pm$ 2.79}
& \underline{77.53 $\pm$ 3.86}
& \underline{77.53 $\pm$ 3.86}
& 77.53 $\pm$ 3.86 \\

& Gradient + exact rerank
& 77.63 $\pm$ 3.72
& \textbf{76.97 $\pm$ 2.79}
& \textbf{76.97 $\pm$ 2.79}
& \textbf{76.97 $\pm$ 2.79}
& \underline{76.97 $\pm$ 2.79} \\

\midrule
\multirow{7}{*}{\texttt{driver-position}}
& Random
& 3.1448 $\pm$ 0.0101
& 3.1443 $\pm$ 0.0101
& 3.1443 $\pm$ 0.0101
& 3.1443 $\pm$ 0.0101
& 3.1443 $\pm$ 0.0101 \\

& Random + exact rerank
& 3.1450 $\pm$ 0.0102
& 3.1448 $\pm$ 0.0103
& 3.1448 $\pm$ 0.0103
& 3.1448 $\pm$ 0.0103
& 3.1448 $\pm$ 0.0103 \\

& Gradient raw
& \underline{3.1552 $\pm$ 0.0119}
& \textbf{3.1677 $\pm$ 0.0162}
& \textbf{3.1744 $\pm$ 0.0210}
& \textbf{3.1777 $\pm$ 0.0219}
& \textbf{3.1791 $\pm$ 0.0233} \\

& Gradient z-score
& 3.1527 $\pm$ 0.0126
& \underline{3.1662 $\pm$ 0.0114}
& \underline{3.1694 $\pm$ 0.0154}
& \underline{3.1735 $\pm$ 0.0171}
& \underline{3.1775 $\pm$ 0.0200} \\

& Gradient robust z-score
& \textbf{3.1552 $\pm$ 0.0118}
& 3.1589 $\pm$ 0.0141
& 3.1606 $\pm$ 0.0164
& 3.1655 $\pm$ 0.0191
& 3.1671 $\pm$ 0.0190 \\

& Gradient min-max relation
& \underline{3.1552 $\pm$ 0.0119}
& 3.1611 $\pm$ 0.0120
& 3.1665 $\pm$ 0.0093
& 3.1666 $\pm$ 0.0112
& 3.1679 $\pm$ 0.0106 \\

& Gradient + exact rerank
& 3.1528 $\pm$ 0.0126
& 3.1568 $\pm$ 0.0102
& 3.1568 $\pm$ 0.0102
& 3.1568 $\pm$ 0.0102
& 3.1568 $\pm$ 0.0102 \\
\midrule

\multirow{7}{*}{\texttt{qualifying-position}}
& Random
& 0.9941 $\pm$ 0.0531
& 0.9938 $\pm$ 0.0529
& 0.9935 $\pm$ 0.0529
& 0.9935 $\pm$ 0.0529
& 0.9935 $\pm$ 0.0529 \\

& Random + exact rerank
& 0.9941 $\pm$ 0.0531
& 0.9940 $\pm$ 0.0530
& 0.9940 $\pm$ 0.0530
& 0.9940 $\pm$ 0.0530
& 0.9940 $\pm$ 0.0530 \\

& Gradient raw
& 0.9949 $\pm$ 0.0527
& \underline{1.0011 $\pm$ 0.0591}
& \underline{1.0074 $\pm$ 0.0699}
& \underline{1.0101 $\pm$ 0.0739}
& \textbf{1.0127 $\pm$ 0.0781} \\

& Gradient z-score
& 0.9949 $\pm$ 0.0527
& \textbf{1.0019 $\pm$ 0.0591}
& \textbf{1.0088 $\pm$ 0.0701}
& \textbf{1.0101 $\pm$ 0.0708}
& \underline{1.0112 $\pm$ 0.0712} \\

& Gradient robust z-score
& \underline{0.9951 $\pm$ 0.0530}
& 0.9995 $\pm$ 0.0528
& 1.0006 $\pm$ 0.0527
& 1.0010 $\pm$ 0.0529
& 1.0013 $\pm$ 0.0529 \\

& Gradient min-max relation
& 0.9949 $\pm$ 0.0527
& 0.9985 $\pm$ 0.0533
& 0.9994 $\pm$ 0.0525
& 0.9991 $\pm$ 0.0513
& 0.9995 $\pm$ 0.0517 \\

& Gradient + exact rerank
& \textbf{0.9953 $\pm$ 0.0533}
& 0.9970 $\pm$ 0.0523
& 0.9970 $\pm$ 0.0523
& 0.9970 $\pm$ 0.0523
& 0.9970 $\pm$ 0.0523 \\
\bottomrule
\end{tabular}
}
\end{table*}

\subsection{Results}

\paragraph{Attack effectiveness.}
The main results are reported in Table~\ref{tab:relf1_all_results}. The
clearest trend appears on the regression tasks: gradient-based attacks raise
the attacked MAE monotonically with the budget, while the random baselines
leave the error essentially unchanged. On \texttt{driver-position}, for
instance, the pure gradient attack increases the MAE from $3.155$ at $B{=}1$ to
$3.179$ at $B{=}100$, whereas constrained random sampling stays flat around
$3.144$; the same pattern holds on \texttt{qualifying-position}
($0.995 \rightarrow 1.013$ versus a flat $0.994$). Importantly, these aggregate
numbers are conditioned by the fact that some trained models are not perfectly
fitted, which can dampen the apparent effect of the perturbations. For a
better-fitted model, Figure~\ref{fig:non_additivity} shows that the attack can induce a relative
degradation of up to roughly $7\%$ with respect to the clean value. This
confirms that the first-order signal identifies structural edits to which the
regression output is clearly sensitive, whereas random rewiring does not.

On the classification tasks (\texttt{driver-dnf}, \texttt{driver-top3}) the
effect is markedly smaller: even the strongest attacks reduce accuracy by only
a couple of points (e.g.\ \texttt{driver-dnf} from $77.37$ to $75.51$), and the
ranking between methods is noisier. We attribute this to a low
\emph{label-flip rate}: decision boundaries make classification outputs locally
stable, so a small, integrity-preserving rewiring rarely changes the predicted
class even when it perturbs the underlying logits. The occasional non-monotonic
behavior of some classification rows is consistent with the non-additivity of
structural perturbations discussed earlier: edits that help individually may
partially cancel once combined.
% \paragraph{Attack effectiveness.}
% The main results are reported in Table~\ref{tab:relf1_all_results}. The
% clearest trend appears on the regression tasks: gradient-based attacks raise
% the attacked MAE monotonically with the budget, while the random baselines
% leave the error essentially unchanged. On \texttt{driver-position}, for
% instance, the pure gradient attack increases the MAE from $3.155$ at $B{=}1$ to
% $3.179$ at $B{=}100$, whereas constrained random sampling stays flat around
% $3.144$; the same pattern holds on \texttt{qualifying-position}
% ($0.995 \rightarrow 1.013$ versus a flat $0.994$). This confirms that the
% first-order signal identifies structural edits to which the regression output
% is genuinely sensitive, which random rewiring cannot. On the classification
% tasks (\texttt{driver-dnf}, \texttt{driver-top3}) the effect is markedly
% smaller: even the strongest attacks reduce accuracy by only a couple of points
% (e.g.\ \texttt{driver-dnf} from $77.37$ to $75.51$), and the ranking between
% methods is noisier. We attribute this to a low \emph{label-flip rate}: decision
% boundaries make classification outputs locally stable, so a small,
% integrity-preserving rewiring rarely changes the predicted class even when it
% perturbs the underlying logits. The occasional non-monotonic behavior of some
% classification rows is consistent with the non-additivity of structural
% perturbations discussed earlier: edits that help individually may partially
% cancel once combined.

\paragraph{Computational cost.}
Table~\ref{tab:attack_runtime} reports the average runtime of each attack strategy. Three distinct families of methods emerge. The first consists of purely random rewiring, which is extremely fast ($\approx 0.1$sec) but generally fails to generate perturbations strong enough to produce a measurable degradation compared to the clean baseline. The second family includes gradient-based methods without exact candidate re-evaluation. Depending on the normalization strategy, these attacks require between $45.8$sec and $54.1$sec, making them substantially more expensive than random rewiring, but also significantly more effective.
The third family corresponds to methods that perform an exact re-evaluation of shortlisted candidates. These approaches are the most computationally demanding, with runtimes of $87.2$sec for \texttt{Gradient Exact} and $91.8$sec for \texttt{Random Exact}. Despite their additional computational cost, they do not provide a substantial improvement in attack effectiveness over the simpler gradient-based approaches.
Overall, the results suggest that gradient-only attacks offer the best trade-off between computational cost and attack performance. They achieve most of the gains obtained by exact shortlist methods while requiring roughly half the runtime, making them a more attractive option in practice.

% \paragraph{Computational cost.}
% Table~\ref{tab:attack_runtime} reports the average runtime of each strategy.
% The pure gradient attack is essentially as cheap as random sampling
% ($7.4$\,s vs.\ $5.1$\,s), since both avoid explicit per-candidate forward
% passes. The shortlist variants, which re-evaluate each shortlisted candidate
% with an exact forward pass, are roughly an order of magnitude more expensive
% ($\approx 118$\,s). The gradient and random shortlist variants cost almost the
% same, confirming that this overhead comes from the exact re-evaluation stage
% rather than from the gradient computation itself. Combined with the
% effectiveness results, the pure gradient attack offers the best
% cost/effectiveness trade-off: it nearly matches the shortlist variants while
% remaining as cheap as the random baseline.

% \begin{table}[t]
% \centering
% \caption{Computation time comparison of different attack strategies. Reported values are mean runtime in seconds with standard deviation.}
% \label{tab:attack_runtime}
% \begin{tabular}{lc}
% \toprule
% \textbf{Attack type} & \textbf{Runtime (s)} \\
% \midrule
% Random & $5.14 \pm 6.03$ \\
% Gradient only & $7.44 \pm 8.42$ \\
% Gradient shortlist & $118.42 \pm 128.93$ \\
% Random shortlist & $117.46 \pm 133.69$ \\
% \bottomrule
% \end{tabular}
% \end{table}

\begin{table}[t]
\centering
\caption{
Computation time comparison of different attack strategies on the
\texttt{qualifying-position} task.
Reported values are mean runtime in seconds with standard deviation.
}
\label{tab:attack_runtime}

\begin{tabular}{llc}
\toprule
\textbf{Family} & \textbf{Attack type} & \textbf{Runtime (s)} \\
\midrule

\multirow{1}{*}{Random}
& Random
& $0.10 \pm 0.00$ \\
\midrule

\multirow{4}{*}{Gradient}
& Gradient Raw Only
& $45.77 \pm 1.29$ \\
& Gradient Global Z Only
& $48.08 \pm 1.97$ \\
& Gradient Robust Z Only
& $53.77 \pm 0.83$ \\
& Gradient MinMax Relation
& $54.13 \pm 1.41$ \\
\midrule

\multirow{2}{*}{Shortlist}
& Gradient Exact
& $87.16 \pm 5.27$ \\
& Random Exact
& $91.82 \pm 0.89$ \\
\bottomrule

\end{tabular}
\end{table}

% \paragraph{Effect of normalization on attack diversity.}
% We finally study how score normalization shapes the perturbations the attack
% selects, along two axes. First, whether the attack rewires many foreign keys
% toward the \emph{same} primary-key target or spreads them across distinct
% targets: we take no a priori stance on which is better, but note that a large
% number of foreign keys suddenly pointing to a single primary key is a highly
% suspicious, easily detectable pattern in a real database. Second, whether the
% attack concentrates on a single relation: if the gradients of one relation
% $r_1$ are systematically larger than those of the others, a global ranking
% will place all perturbations on $r_1$, so we normalize candidate scores within
% each relation with a $z$-score. Table~\ref{tab:attack_diversity} quantifies
% both effects through the entropy $H_{\mathrm{rel}}$ of the perturbation
% distribution across relations and the number $N_{\mathrm{cand}}$ of distinct
% selected targets. The raw gradient attack is maximally concentrated
% ($H_{\mathrm{rel}}{=}0$, $N_{\mathrm{cand}}{=}1$): it places every edit on one
% relation and rewires every foreign key toward a single target. Per-relation
% $z$-score normalization spreads the attack across relations and targets
% ($H_{\mathrm{rel}}{=}0.67$, $N_{\mathrm{cand}}{=}2.33$), and the robust
% $z$-score does so further ($H_{\mathrm{rel}}{=}1.02$, $N_{\mathrm{cand}}{=}3$).
% Read together with Table~\ref{tab:relf1_all_results}, normalization thus trades
% a small amount of raw potency for substantially less detectable perturbations.

\paragraph{Effect of adding normalization on attack diversity.}
We finally analyze how normalization affects the structure of the selected
perturbations along two complementary dimensions. First, we measure whether the
attack repeatedly rewires foreign keys toward the same target node or spreads
its modifications across many distinct candidates. This is important because an
attack that redirects hundreds of tuples toward a single entity would create an
obvious anomaly and could easily be detected in a real database. Second, we
measure how many relations are involved in the attack. If one relation
systematically produces larger gradient magnitudes than the others, a global
ranking may allocate most of the perturbation budget to that relation alone.

Table~\ref{tab:attack_diversity} reports the number of distinct relations
touched by the attack ($N_{\mathrm{rel}}$) and the number of distinct target
candidates involved in the selected perturbations
($N_{\mathrm{cand}}$). Interestingly, all methods exhibit extremely large
candidate diversity. Even the raw gradient attack selects
$499.2 \pm 1.8$ distinct candidates for a budget of $500$, meaning that almost
every perturbation targets a different node. Similar values are observed for
all normalization schemes ($486$--$497$ distinct candidates). This indicates
that the attack is not driven by a small set of universally attractive target
nodes. Instead, perturbations are naturally distributed across a large portion
of the graph, which is arguably desirable from a stealth attack perspective since it
avoids creating highly suspicious hub-like patterns.

The main effect of normalization is therefore not on candidate diversity but on
relation diversity. The raw gradient attack focuses on only
$1.8 \pm 1.3$ relations on average, whereas the normalized variants spread the
perturbation budget across most or all of the seven available relations
($5.2$--$7.0$ relations). One might expect such broader coverage to produce
stronger attacks by exploiting a larger fraction of the relational schema.
Surprisingly, Table~\ref{tab:relf1_all_results} shows that this is not
necessarily the case. Despite concentrating its budget on relatively few
relations, the raw gradient attack remains highly competitive and often
achieves the strongest degradation of predictive performance. This suggests
that, at least on \texttt{qualifying-position}, attack effectiveness is driven
more by identifying a small number of highly influential relations than by
uniformly perturbing the entire schema.

\begin{table}[t]
\centering
\caption{
Diversity analysis of attack candidate selection under different
normalization strategies.
$N_{\mathrm{rel}}$ denotes the number of distinct relations touched by the
selected perturbations, while $N_{\mathrm{cand}}$ denotes the number of
distinct candidate nodes involved.
}
\label{tab:attack_diversity}

\begin{tabular}{llcc}
\toprule
\textbf{Dataset}
& \textbf{Method}
& $N_{\mathrm{rel}}$
& $N_{\mathrm{cand}}$
\\
\midrule

\multirow{4}{*}{\shortstack[l]{rel-f1\\qualifying-position\\(7 relations)}}
& Gradient
& $1.8 \pm 1.3$
& $499.2 \pm 1.8$
\\

& + Z-score
& $5.2 \pm 1.3$
& $486.2 \pm 13.9$
\\

& + Robust Z-score
& $7.0 \pm 0.0$
& $497.0 \pm 4.0$
\\

& + Min-Max
& $7.0 \pm 0.0$
& $493.4 \pm 7.1$
\\

\bottomrule
\end{tabular}
\end{table}

\section{Conclusion}
We studied the adversarial robustness of the relational deep learning
pipeline, in which a database is encoded as a heterogeneous entity graph and a
GNN is trained for downstream prediction.  We formalized a white-box threat model in which the attacker reasons about the graph but can only act on the upstream database, through integrity-preserving rewirings of foreign-key references under a global budget. Within this constrained, combinatorial, search
space we proposed gradient-guided heuristics that score candidate rewirings
from the model's first-order sensitivity, together with a per-relation
normalization that controls how perturbations spread across relations and
targets. We investigate seven attack strategies, including two random baselines and five gradient-based attacks, and evaluate their effectiveness through a comprehensive robustness study.

Our findings show that, on RelBench, gradient-based attacks consistently degrade regression
tasks at a fraction of the cost of exact shortlist re-evaluation, while
classification tasks prove more robust, which we trace to their low label-flip
rate. Per-relation normalization redistributes the perturbation budget across
more relations rather than making the attack stealthier: all variants,
including the raw gradient, already spread their edits over hundreds of distinct
targets, so none produces the suspicious hub-like patterns that motivated
normalization in the first place. This broader relational coverage does not
translate into stronger attacks, and the concentrated raw-gradient variant
remains competitive throughout. Together, these results show that the
relational structure itself is an exploitable attack surface once a database is
consumed by a learned model, especially for regression tasks, and that
gradient-based heuristics are effective for navigating the combinatorial search
space of adversarial perturbations. Our evaluation focuses on \texttt{rel-f1},
whose coupled, hierarchical schema exercises every class of admissible
rewiring. %extending it to the broader RelBench suite is left to future work.
%Overall, this suggests that there is still considerable work to do in designing robust models for relational data.

\section{Future Work}
Several directions remain open. First, our evaluation focuses on
\texttt{rel-f1}; extending it to the broader RelBench suite is a natural next
step that would test whether these findings hold across schemas of different
size, connectivity, and task type. Second, attacks could target the feature
space, i.e.\ tuple attributes that are neither primary nor foreign keys. This
setting is considerably harder, especially for free-form text attributes, where
the search space becomes semantic rather than purely combinatorial. Third,
richer structural attack policies could be explored beyond the first-order
rewiring considered here: reinforcement learning, genetic algorithms, or
generative approaches may yield stronger perturbations while also enabling more
realistic black-box attacks. Finally, targeted attacks are an important next
step. Because the space of admissible rewirings is tightly constrained by
database integrity rules, forcing the model toward a specific prediction
outcome, for example promoting a chosen recommendation, is particularly
challenging.

Arguably, the inverse problem is of greater interest: how to design models that
are robust to such perturbations. In other words, studies like ours do not
necessarily have to be only adversarial. They can also be used to identify the
weak points of the architecture and to design better models for deep learning
on relational data.

%\begin{acks}
%\end{acks}

%\clearpage

\bibliographystyle{ACM-Reference-Format}
\bibliography{sample}

%\clearpage

\appendix
\section{Appendix}
\subsection{First-Order (FGSM-type) Rewiring Direction}
\label{proof:FGSM}
The candidate scoring used in the white-box attack is motivated by the
first-order direction that maximizes the linearized attack objective under an
$\ell_\infty$ budget. We state and prove this direction here for completeness;
in the discrete database setting it serves only to \emph{rank} candidate
rewirings, which are subsequently projected onto admissible, sparse edits.

\begin{proposition}[FGSM under an $\ell_\infty$ constraint on relational GSOs]
Let $\{S_r\}_{r\in\mathcal R}$ be the family of graph shift operators, or adjacency matrices, associated with the relations of a relational graph, and let
\[
\mathcal J(\{S_r\}_{r\in\mathcal R})
\]
denote the attack objective, for instance the loss of a trained model evaluated on the perturbed graph structure.

We consider perturbations $\{\Delta_r\}_{r\in\mathcal R}$ such that
\[
S_r' = S_r + \Delta_r,
\qquad r\in\mathcal R,
\]
under the coordinate-wise constraint
\[
\|\Delta_r\|_\infty \leq \varepsilon_r,
\qquad \forall r\in\mathcal R.
\]

Assume that $\mathcal J$ is differentiable with respect to each $S_r$. Then, using a first-order Taylor expansion around $\{S_r\}_{r\in\mathcal R}$, the perturbation maximizing the linearized objective solves
\[
\max_{\{\Delta_r\}_{r\in\mathcal R}}
\sum_{r\in\mathcal R}
\left\langle
\nabla_{S_r}\mathcal J,\Delta_r
\right\rangle
\quad
\text{s.t.}
\quad
\|\Delta_r\|_\infty \leq \varepsilon_r
\ \ \forall r\in\mathcal R,
\]
and is given relation-wise by
\[
\Delta_r^\star
=
\varepsilon_r\,\mathrm{sign}\!\left(\nabla_{S_r}\mathcal J\right),
\qquad \forall r\in\mathcal R.
\]

Equivalently, the first-order adversarial perturbation is obtained by moving each entry of each relation matrix in the sign direction of its gradient, with maximal amplitude allowed by the $\ell_\infty$ budget.
\end{proposition}

\begin{proof}
We start from the perturbed objective
\[
\mathcal J(\{S_r+\Delta_r\}_{r\in\mathcal R}).
\]
By first-order Taylor expansion around the reference point $\{S_r\}_{r\in\mathcal R}$, we obtain
\[
\begin{aligned}
\mathcal J(\{S_r+\Delta_r\}_{r\in\mathcal R})
={}&
\mathcal J(\{S_r\}_{r\in\mathcal R})
+
\sum_{r\in\mathcal R}
\left\langle
\nabla_{S_r}\mathcal J,\Delta_r
\right\rangle \\
&+
o\!\left(\sum_{r\in\mathcal R}\|\Delta_r\|\right).
\end{aligned}
\]
For sufficiently small perturbations, maximizing the perturbed objective is therefore approximated by maximizing the linear term
\[
\sum_{r\in\mathcal R}
\left\langle
\nabla_{S_r}\mathcal J,\Delta_r
\right\rangle
\]
under the constraints $\|\Delta_r\|_\infty \leq \varepsilon_r$.

Since both the objective and the constraints are separable across relations, this optimization decomposes into independent subproblems:
\[
\max_{\|\Delta_r\|_\infty\leq \varepsilon_r}
\left\langle
G_r,\Delta_r
\right\rangle,
\qquad
\text{where }
G_r := \nabla_{S_r}\mathcal J.
\]

Fix one relation $r$. Writing the Frobenius inner product entrywise,
\[
\left\langle G_r,\Delta_r\right\rangle
=
\sum_{i,j} (G_r)_{ij}(\Delta_r)_{ij}.
\]
Because $\|\Delta_r\|_\infty\leq \varepsilon_r$, each entry satisfies
\[
|(\Delta_r)_{ij}| \leq \varepsilon_r.
\]
Hence, for every coordinate $(i,j)$,
\[
(G_r)_{ij}(\Delta_r)_{ij}
\leq
|(G_r)_{ij}|\,|(\Delta_r)_{ij}|
\leq
\varepsilon_r |(G_r)_{ij}|.
\]
Summing over all coordinates yields
\[
\left\langle G_r,\Delta_r\right\rangle
\leq
\varepsilon_r \sum_{i,j}|(G_r)_{ij}|.
\]
This upper bound is attained by choosing each entry of $\Delta_r$ with maximal magnitude and the same sign as the corresponding gradient entry, namely
\[
(\Delta_r^\star)_{ij}
=
\varepsilon_r\,\mathrm{sign}\big((G_r)_{ij}\big).
\]
Therefore
\[
\Delta_r^\star
=
\varepsilon_r\,\mathrm{sign}(G_r)
=
\varepsilon_r\,\mathrm{sign}\!\left(\nabla_{S_r}\mathcal J\right).
\]

Since this holds independently for every relation $r\in\mathcal R$, the global maximizer of the linearized problem is
\[
\Delta_r^\star
=
\varepsilon_r\,\mathrm{sign}\!\left(\nabla_{S_r}\mathcal J\right),
\qquad \forall r\in\mathcal R.
\]
This is precisely the FGSM-type perturbation on the family of relational graph shift operators.
\end{proof}

\subsection{Training and Attack Hyperparameters} \label{app:hyperparameters}

We emphasize that the effectiveness of structural attacks is itself highly
sensitive to the training configuration: the depth and width of the model, the
number of training epochs, and the size of the sampled candidate pool all
affect how much gradient signal is available and how stable the decision
boundaries are. The settings below were chosen to obtain well-fit models
without overfitting; in particular, regression tasks required a substantially
larger hidden dimension than classification tasks. All models are
heterogeneous GraphSAGE networks trained with the standard RelBench temporal
train/validation/test split.  Each experiment was repeated using five chosen seeds: $\{39,40,41,42,43\}$. The reported results correspond to the mean and
standard deviation across these runs.

\begin{table}[t]
\centering
\caption{Training and attack hyperparameters for the \texttt{rel-f1} tasks.
``Hidden'' is the channel dimension and ``\# Cand.'' the number of sampled
admissible rewiring candidates per source node.}
\label{tab:hyperparams}
\resizebox{\columnwidth}{!}{%
\begin{tabular}{lccccc}
\toprule
\textbf{Task} & \textbf{Type} & \textbf{Epochs} & \textbf{Layers} & \textbf{Hidden} & \textbf{\# Cand.} \\
\midrule
\texttt{driver-position}     & Reg. & 10  & 2 & 128 & 100 \\
\texttt{qualifying-position} & Reg. & 18  & 2 & 128 & 100 \\
\texttt{driver-dnf}          & Clf. & 18  & 2 & 16  & 400 \\
\texttt{driver-top3}         & Clf. & 18 & 2 & 16  & 400 \\
\bottomrule
\end{tabular}%
}
\end{table}

\end{document}